\documentclass[submission,copyright,creativecommons,a4paper]{article}

\usepackage[numbers]{natbib}
\usepackage{amsmath}
\usepackage{amssymb}
\usepackage{amsthm}
\usepackage{xcolor}
\usepackage{listings}
\usepackage{mathpartir}
\usepackage{stmaryrd}
\usepackage{xspace}
\usepackage{iftex}
\usepackage{a4wide}

\theoremstyle{plain}
\newtheorem{theorem}{Theorem}[section]
\newtheorem{lemma}[theorem]{Lemma}
\newtheorem{proposition}[theorem]{Proposition}
\newtheorem{corollary}[theorem]{Corollary}

\theoremstyle{definition}
\newtheorem{definition}[theorem]{Definition}
\newtheorem{remark}[theorem]{Remark}

\ifpdf
  \usepackage{underscore}         % Only needed if you use pdflatex.
  \usepackage[T1]{fontenc}        % Recommended with pdflatex
\else
  \usepackage{breakurl}           % Not needed if you use pdflatex only.
\fi

\newcommand{\ppar}{\mathbin{\mid}}
\newcommand{\restr}[1]{(\nu\,#1)\;}
\newcommand{\nil}{\mathbf{0}}
\newcommand{\trans}[1]{\xrightarrow{#1}}
\newcommand{\bisim}{\sim}
\newcommand{\piCalc}{\pi\text{-calculus}}

\newcommand{\IntentT}[5]{\mathsf{Intent}\langle #1,\,#2,\,#3,\,#4,\,#5\rangle}
\newcommand{\SlotT}[3]{\mathsf{Slot}\langle #1,\,#2,\,#3\rangle}
\newcommand{\CollectSlotT}[2]{\mathsf{CollectSlot}\langle #1,\,#2\rangle}
\newcommand{\ExecuteT}[2]{\mathsf{Execute}\langle #1,\,#2\rangle}
\newcommand{\ResultT}[1]{\mathsf{Result}\langle #1\rangle}
\newcommand{\ErrorT}[2]{\mathsf{Error}\langle #1,\,#2\rangle}

\newcommand{\ToolT}[3]{\mathsf{Tool}\langle #1,\,#2,\,#3\rangle}
\newcommand{\ResourceT}[2]{\mathsf{Resource}\langle #1,\,#2\rangle}
\newcommand{\PromptT}[2]{\mathsf{Prompt}\langle #1,\,#2\rangle}
\newcommand{\ToolsListT}[1]{\mathsf{ToolsList}\langle #1\rangle}
\newcommand{\ToolCallT}[2]{\mathsf{ToolCall}\langle #1,\,#2\rangle}
\newcommand{\ValidateT}[2]{\mathsf{Validate}\langle #1,\,#2\rangle}

\newcommand{\SGD}{\ensuremath{\mathsf{SGD}}\xspace}
\newcommand{\MCP}{\ensuremath{\mathsf{MCP}}\xspace}
\newcommand{\MCPplus}{\ensuremath{\mathsf{MCP}^{+}}\xspace}
\newcommand{\Phimap}{\Phi}
\newcommand{\PhiInv}{\Phi^{-1}}
\newcommand{\PhiPlus}{\Phi^{+}}
\newcommand{\PhiPlusInv}{(\Phi^{+})^{-1}}

\newcommand{\ainvoke}[2]{\mathsf{invoke}(#1,\,#2)}
\newcommand{\acall}[2]{\mathsf{call}(#1,\,#2)}
\newcommand{\acollect}[2]{\mathsf{collect}(#1,\,#2)}
\newcommand{\aexec}{\mathsf{execute}}
\newcommand{\ares}[1]{\mathsf{result}(#1)}

\newcommand{\alist}[1]{\mathsf{list}(#1)}
\newcommand{\aerr}[2]{\mathsf{error}(#1,\,#2)}
\newcommand{\aread}[1]{\mathsf{read}(#1)}
\newcommand{\avalidate}{\mathsf{validate}}

\title{Formal Semantics for Agentic Tool Protocols: A Process Calculus Approach}

\author{Andreas Schlapbach, Swiss Federal Railways (SBB)\\ SBB IT Berne,
Switzerland\\ schlpbch@gmail.com\footnote{\copyright\ 2026 Andreas Schlapbach.
Research supported by SBB-IT. The views expressed in this paper are the
author's own and do not necessarily reflect the views or policies of SBB.
ORCID: 0009-0006-2329-2626.}}

\date {June, 25th, 2026}

\begin{document}

\maketitle

\begin{abstract}
The emergence of large language model agents capable of invoking
external tools has created urgent need for formal verification of agent
protocols.  Two paradigms dominate this space: \emph{Schema-Guided
Dialogue} (\SGD), a research framework for zero-shot API generalization,
and the \emph{Model Context Protocol} (\MCP), an industry standard for
agent--tool integration.  While both enable dynamic service discovery
through schema descriptions, their formal relationship remains
unexplored.  We present the first process calculus formalization of \SGD and \MCP,
proving they are structurally bisimilar under a well-defined mapping
$\Phimap$.  We demonstrate that the reverse mapping $\PhiInv$
is partial and lossy, revealing critical gaps in \MCP's expressivity.
Through bidirectional analysis, we identify four principles---semantic
completeness, explicit action boundaries, failure mode documentation,
and inter-tool relationship declaration---as necessary and sufficient 
conditions for full behavioral equivalence.  We formalize these principles 
as type-system extensions $\MCPplus$, proving $\MCPplus \cong \SGD$. Our 
work provides the first formal foundation for verified agent systems and
establishes schema quality as a provable safety property. Practically, 
this means that the current \MCP specification has expressiveness gaps
compared to \SGD and would benefit from the proposed extensions.
\end{abstract}

\section{Introduction}

The rapid deployment of large language model (LLM)
agents~\cite{transformers,gpt3} in production systems has created a critical
gap: while these agents---capable of multi-step
reasoning~\cite{chain-of-thought} and reactive tool
use~\cite{react}---orchestrate complex workflows across services, we lack
formal methods to verify their correctness, safety, and behavioral properties.
Two paradigms have emerged to address agent--tool integration:

\subsection{The Schema-Guided Dialogue Systems (\SGD)}

The SGD dataset~\cite{sgd-aaai} introduced a paradigm shift in task-oriented
dialogue: rather than training on fixed ontologies, SGD models learn to
interpret arbitrary API schemas at runtime.  Key components are \emph{intents}
(high-level user goals), \emph{slots} (parameters to collect), \emph{schemas}
(JSON definitions with natural-language descriptions), and \emph{state
tracking} (frame-based representation across turns). Dialogue state tracking
has a longer history in shared tasks~\cite{dstc}, and BERT-based
encoders~\cite{bert} enabled strong cross-service generalization.  Fast and
robust SGD trackers~\cite{fastsgd} demonstrated practical viability at scale.
More recently, in-context learning has shown that few demonstration examples
suffice to adapt to new schemas---mirroring the zero-shot generalization that
SGD pioneered \cite{bapna2017zeroshot, palatucci2009zeroshot}.

\subsection{Model Context Protocol (\MCP)}

The \MCP is an open protocol that enables seamless integration between LLM
applications and external data sources and tools.  The architecture comprises a
\emph{host} (LLM environment), a \emph{client} (protocol handler), and a
\emph{server}. Servers provide the fundamental building blocks for adding
context to language models via \MCP. These primitives enable rich interactions
between clients, servers, and language models. It standardizes agent--tool
communication through three primitives: Firstly, \emph{Tools} enable models to
interact with external systems, such as querying databases, calling APIs, or
performing computations. Each tool is uniquely identified by a name and
includes metadata describing its schema. Secondly, \emph{Resources} allow
servers to share data that provides context to language models, such as files,
database schemas, or application-specific information. Each resource is
uniquely identified by a URI. Thirdly, \emph{Prompts} allow servers to provide
structured messages and instructions for interacting with language models.
Clients can discover available prompts, retrieve their contents, and provide
arguments to customize them.
	  	
\subsection{Shared Core Principles}

Both approaches share a core principle: agents discover and reason about
services through machine-readable schemas without retraining. This analysis is
grounded in experience with an experimental ecosystem of over 10 agents. The
agents are each a domain expert, following domain-driven design
\cite{vernon-implementing-ddd}. The agents  are intelligently federated,
coordinating through dynamically discovered schema relationships and the \MCP
rather than rigid orchestration logic. This federated architecture has revealed
critical gaps between the theoretical elegance of \SGD principles and the
practical requirements of production agent systems: over 1,000 tool-to-tool
dependencies must be managed, action boundaries must be enforced
deterministically, and failure modes must be recoverable without human
intervention. 

The \emph{formal} relationship between the two approaches remains unexplored.
Can we prove they are equivalent?  What properties are preserved under
transformation?  What gaps exist in either framework?  This paper's aim is to
provide answers to these questions by grounding the analysis in process
calculus.

\subsection{Contributions}

This paper makes the following contributions:

\begin{enumerate}
  \item \textbf{Formal semantics for \SGD and \MCP} using $\piCalc$,
    defining their syntax, operational semantics, and labeled transition
    systems (Section~\ref{sec:formalization}).
  \item \textbf{Bisimulation proof} showing $\SGD \bisim \MCP$ under mapping
    $\Phimap$, establishing structural and behavioral equivalence
    (Section~\ref{sec:forward}).
  \item \textbf{Gorla-Style analysis} of $\Phi$ showing that while it is a valid
    encoding it is neither surjective nor fully abstract
    (Section~\ref{sec:reverse-gorla-style}).
  \item \textbf{Four principles as type system} formalizing necessary
    conditions for agent safety and schema quality. Safety can also be regarded
    as respecting a user's privacy by asking for permission before an action
    (Section~\ref{sec:principles}).
  \item \textbf{Extended calculus $\MCPplus$} proving that adding four
    principles creates full bijection: $\MCPplus \cong \SGD$
    (Section~\ref{sec:full-equiv}).
\end{enumerate}

\section{Background and Related Work}
\label{sec:backgrou-related-work}

Consider a ticket booking agent which needs to book a train ticket. The agent needs first to check the user using its id. Then has has to validate if the user has a valid account and enough balance to book the ticket. If the user has enough balance, the agent can book then book the ticket from an origin to a destination at a given date and travel class. The agent needs to send to print the ticket or send it via email.

If the agent is an AI agent, the questions requiring formal answers include: 
Needs the agent to \emph{always} fetch the user? Can the third step of booking the ticket execute without the second step of checking the balance? What happens if the third step fails after debiting the account?

Current approaches to validate a complex ecosystems of agents rely on testing
(incomplete coverage), prompt engineering~\cite{chain-of-thought} (no
guarantees, brittle), or human review (not scalable).  Agent frameworks
grounded in BDI theory~\cite{rao1995bdi} offer structured reasoning, but stop
short of protocol-level verification.  Formal semantics enable
\emph{verification} (proving safety for all executions), \emph{composition}
(building complex systems from verified components), \emph{optimization}
(preserving equivalence under transformation), and \emph{debugging} (tracing
violations to specific protocol violations).  \subsection{Process Calculi for
Concurrency}

Process calculi provide mathematical frameworks for reasoning about concurrent,
communicating systems. The $\piCalc$~\cite{milner1999pi} extends
CCS~\cite{milner1989ccs} with mobile channels, enabling dynamic network
topologies; essential for agent systems where tool availability changes at
runtime. 

Related work includes session types~\cite{honda1998session} for protocol
verification, choreography calculi~\cite{carbone2012choreographies} for
multi-party workflows, and probabilistic process algebras for stochastic
systems~\cite{dargenio2022probabilistic}.  However, all these frameworks assume
static protocols. Agent systems require runtime schema interpretation. Work by
\cite{allegrini2026formalizingsafetysecurityfunctional} specifies temporal
logic properties for a host agent architecture covering \MCP and A2A.
\cite{odersky2026trackingcapabilitiessaferagents} take a different stance: they
make the medium safe by construction through a type system; our work identifies
the necessary and sufficient principles for such a type system to achieve full
expressivity.

Highly relevant but somewhat orthogonal is the work on the LLMbda Calculus by
\cite{garby2026lambdacalculus} by providing a principled semantic foundation
for reasoning about an agent's behavior and safety. The work addresses this gap
by introducing an untyped call-by-value lambda calculus enriched with dynamic
information-flow control and a small number of primitives for constructing
prompt-response conversations. The calculus faithfully represents planner loops
and their vulnerabilities, including the mechanisms by which prompt injection
alters subsequent computation.

\section{Process Calculus Formalization}
\label{sec:formalization}

In this paper we formalize \SGD and \MCP as communicating processes in
$\piCalc$, defining their syntax, operational semantics, and observable
behaviors. The $\piCalc$ was chosen because: First, it makes topology dynamic
where processes can acquire new channel names at runtime and, secondly, that we
are accustomed to it through previous work on \textsc{Piccola}, a pure
composition language based on this calculus \cite{ALSN01} \cite{Ach02}, \cite{schlapbach03whiteboxreuse}. 

\subsection{\SGD Process Calculus}

The Schema-Guided Dialogue (\SGD) dataset\footnote{The data set is publicly
available at:
https://github.com/google-research-datasets/dstc8-schema-guided-dialogue}
consists of over 20k annotated multi-domain, task-oriented conversations
between a human and a virtual assistant. These conversations involve
interactions with services and APIs spanning 20 domains, such as banks, events,
media, calendar, travel, and weather. Additionally, the dataset contains
unseen domains and services in the evaluation set to quantify the performance
in zero-shot or few-shot settings.

\subsubsection{Syntax}

The \SGD data set has a well defined structure. Dialogues are represented as a
list of turns, where each turn contains either a user or system utterance. The
annotations for a turn are grouped into frames, where each frame corresponds to
a single service. Each turn in the single domain dataset contains exactly one
frame. In multi-domain datasets, some turns may have multiple frames.

Based on this structure we propose a possible grammar, where $n$ is a name
(string), $d$ a natural-language description, $R$ a list of required slots, $O$
a list of optional slots, and $t \in \{\mathsf{true}, \mathsf{false}\}$ a
transactionality flag:

\begin{align*}
S \;::=\; &
  \IntentT{n}{d}{R}{O}{t}                  &&\text{intent with metadata} \\
  \mid\; &
  \SlotT{\mathit{name}}{\mathit{type}}{\mathit{values}} &&\text{slot definition} \\
  \mid\; &
  \CollectSlotT{s}{v}\mathbin{.}S          &&\text{slot collection, then $S$} \\
  \mid\; &
  \ExecuteT{\mathit{intent}}{\mathit{bindings}} &&\text{API invocation} \\
  \mid\; &
  \ResultT{\mathit{output}}                &&\text{return value} \\
  \mid\; &
  \ErrorT{\mathit{type}}{\mathit{msg}}     &&\text{error state} \\
  \mid\; &
  S_1 \ppar S_2                            &&\text{parallel composition} \\
  \mid\; &
  \restr{c}S                               &&\text{channel restriction} \\
  \mid\; &
  {!}S                                     &&\text{replication} \\
  \mid\; &
  \nil                                     &&\text{null process}
\end{align*}

\subsubsection{Example: Train Booking in \SGD}

The example illustrates the intent to book a train, where origin, destination
and date need to be set and a confirmation is required before booking.

\begin{lstlisting}
BookTrainIntent = Intent<
  name = "BookTrain",
  description = "Books a train ticket",
  R = [Slot<"from", "string", ["8507000", "8508050"]>,
       Slot<"to", "string", ["8507000", "8508050"]>,
       Slot<"date", "date", []>],
  O = [Slot<"class", "string", ["economy", "business"]>],
  t = true  // Transaction requires confirmation
>
\end{lstlisting}

\subsubsection{Operational Semantics}
The transition rules for \SGD processes are given in Figure~\ref{fig:sgd-sem}.

\begin{figure}[h]
\begin{mathpar}
\inferrule*[Right=SGD-Invoke-Ok]
  {\forall r \in R.\; r \in \mathrm{dom}(\mathit{params})}
  {\IntentT{n}{d}{R}{O}{t}
    \;\trans{\ainvoke{n}{\mathit{params}}}\;
   \ExecuteT{n}{\mathit{params}}}
\\
\inferrule*[Right=SGD-Invoke-Err]
  {\exists r \in R.\; r \notin \mathrm{dom}(\mathit{params})}
  {\IntentT{n}{d}{R}{O}{t}
    \;\trans{\ainvoke{n}{\mathit{params}}}\;
   \ErrorT{\mathsf{MissingSlots}}{R \setminus \mathit{params}}}
\\
\inferrule*[Right=SGD-Collect]
  { }
  {\CollectSlotT{s}{v}\mathbin{.}S
    \;\trans{\mathsf{collect}(s,v)}\;
   S[s \mapsto v]}
\\
\inferrule*[Right=SGD-Execute-Tx]
  {\substack{t = \mathsf{true} \\[2pt]
   \mathsf{require\_approval} \\[2pt]
   \mathit{output} = \mathsf{invoke\_api}(n,\mathit{b})}}
  {\ExecuteT{n}{\mathit{b}}
    \;\trans{\aexec}\;
   \ResultT{\mathit{output}}}
\\
\inferrule*[Right=SGD-Execute]
  {\substack{t = \mathsf{false} \\[2pt]
   \mathit{output} = \mathsf{invoke\_api}(n,\mathit{b})}}
  {\ExecuteT{n}{\mathit{b}}
    \;\trans{\aexec}\;
   \ResultT{\mathit{output}}}
\\
\inferrule*[Right=SGD-Par]
  {S_1 \;\trans{\alpha}\; S_1'}
  {S_1 \ppar S_2 \;\trans{\alpha}\; S_1' \ppar S_2}
\\
\inferrule*[Right=SGD-Res]
  {\substack{S \;\trans{\alpha}\; S' \\[2pt] c \notin \alpha}}
  {\restr{c}S \;\trans{\alpha}\; \restr{c}S'}
\end{mathpar}
\caption{Operational semantics for \SGD processes.}
\label{fig:sgd-sem}
\end{figure}

\subsection{\MCP Process Calculus}

\subsubsection{Syntax}
The set of \MCP processes is defined by the following grammar, where
$\mathit{schema}$ is a JSON Schema object and $\mathit{uri}$ a resource identifier\footnote{With the latest release of the \MCP
specification the protocol is sessionless and SEP-2567: Sessionless \MCP via Explicit State Handles}:

\begin{align*}
M \;::=\; &
  \ToolT{n}{d}{\mathit{schema}}             &&\text{tool definition} \\
  \mid\; &
  \ResourceT{\mathit{uri}}{\mathit{content}} &&\text{read-only resource} \\
  \mid\; &
  \PromptT{\mathit{template}}{\mathit{args}} &&\text{workflow template} \\
  \mid\; &
  \ToolsListT{\mathit{metadata}}            &&\text{discovery response} \\
  \mid\; &
  \ToolCallT{n}{\mathit{params}}            &&\text{tool invocation} \\
  \mid\; &
  \ValidateT{\mathit{params}}{\mathit{schema}} &&\text{parameter validation} \\
  \mid\; &
  \ExecuteT{n}{\mathit{params}}             &&\text{tool execution } \\
  \mid\; &
  \ResultT{\mathit{output}}                 &&\text{return value} \\
  \mid\; &
  \ErrorT{\mathit{type}}{\mathit{msg}}      &&\text{error state} \\
  \mid\; &
  M_1 \ppar M_2                              &&\text{parallel composition} \\
  \mid\; &
  \restr{c}M                                 &&\text{channel restriction} \\
  \mid\; &
  {!}M                                       &&\text{replication} \\
  \mid\; &
  \nil                                       &&\text{null process}
\end{align*}

\subsubsection{Example: Train Booking in \MCP}

This example invokes the tool \texttt{bookTrainTool}. In \MCP the fact that the
invocation needs a confirmation by the user and that \texttt{class} is an
optional parameter can only be expressed in the description text.

\begin{lstlisting}
BookTrainTool = Tool<
  name = "book_train",
  description = "Books a train ticket, needs confirmation",
  schema = {
    "type": "object",
    "required": ["from", "to", "date"],
    "properties": {
      "from": {"type": "string", "description": "Origin Code"},
      "to": {"type": "string", "description": "Destination Code"},
      "date": {"type": "string", "description": "Date of Travel"},
      "class": {"type": "string", "description": "Class (default: economy)"}
    }
  }
>
\end{lstlisting}

\subsubsection{Operational Semantics}
The transition rules for \MCP processes are given in Figure~\ref{fig:mcp-sem}.

\begin{figure}[h]
\begin{mathpar}
\inferrule*[Right=MCP-Discover]
  { }
  {\ToolsListT{\mathit{tools}}
    \;\trans{\alist{f}}\;
   \{t \in \mathit{tools} \mid \mathsf{matches}(t,f)\}} 
\\
\inferrule*[Right=MCP-Call]
  { }
  {\ToolT{n}{d}{\mathit{schema}}
    \;\trans{\acall{n}{\mathit{params}}}\;
   \ValidateT{\mathit{params}}{\mathit{schema}}}
\\
\inferrule*[Right=MCP-Validate-Ok]
  {\mathsf{conforms}(\mathit{params}, \mathit{schema})}
  {\ValidateT{\mathit{params}}{\mathit{schema}}
    \;\trans{\tau}\;
   \ExecuteT{n}{\mathit{params}}}
\\
\inferrule*[Right=MCP-Validate-Err]
  {\neg\,\mathsf{conforms}(\mathit{params}, \mathit{schema})}
  {\ValidateT{\mathit{params}}{\mathit{schema}}
    \;\trans{\tau}\;
   \ErrorT{\mathsf{ValidationError}}{\mathit{violations}}}
\\
\inferrule*[Right=MCP-Execute]
  {\mathit{output} = \mathsf{invoke\_tool}(n,\mathit{params})}
  {\ExecuteT{n}{\mathit{params}}
    \;\trans{\aexec}\;
   \ResultT{\mathit{output}}}
\\
\inferrule*[Right=MCP-Resource]
  { }
  {\ResourceT{\mathit{uri}}{\mathit{content}}
    \;\trans{\aread{\mathit{uri}}}\;
   \ResultT{\mathit{content}}}
\end{mathpar}
\caption{Operational semantics for \MCP processes.}
\label{fig:mcp-sem}
\end{figure}

\subsection{Labeled Transition Systems}

For both \SGD and \MCP we define a labeled transition system. In both
systems,  the parameter $n$ denotes the tool/intent name, $p$ denotes the
parameter, $t$ denotes the error type, $m$ the error message and $o$ the object
type.

\begin{definition}[\SGD LTS]
The labeled transition system for $\SGD$ is the tuple\linebreak[4]
$(\mathit{States}_\SGD,\,\mathit{Labels}_\SGD,\,\to_\SGD)$ where:
\begin{align*}
  \mathit{States}_\SGD &= \{S \mid S \text{ is a well-formed SGD process}\} \\
  \mathit{Labels}_\SGD &= \begin{aligned}[t]
    \{&\ainvoke{n}{p},\; \mathsf{collect}(s,v),\;\\
     &\aexec,\; \ares{o},\; \aerr{t}{m}\}
    \end{aligned} \\
  \to_\SGD &\;\subseteq\;
    \mathit{States}_\SGD \times \mathit{Labels}_\SGD \times \mathit{States}_\SGD
\end{align*}
where $s$ denotes the slot and $v$ denotes the values.
\end{definition}

\begin{definition}[MCP LTS]
The labeled transition system for $\MCP$ is the tuple\linebreak[4]
$(\mathit{States}_\MCP,\,\mathit{Labels}_\MCP,\,\to_\MCP)$ where:
\begin{align*}
  \mathit{States}_\MCP &= \{M \mid M \text{ is a well-formed MCP process}\} \\
  \mathit{Labels}_\MCP &= \begin{aligned}[t]
     \{&\acall{n}{p},\; \alist{f},\; \tau,\; \aexec,\;\\
     &\ares{o},\; \aerr{t}{m},\; \aread{u}\}
    \end{aligned} \\
  \to_\MCP &\;\subseteq\;
    \mathit{States}_\MCP \times \mathit{Labels}_\MCP \times \mathit{States}_\MCP
\end{align*}
where $u$ denotes a universal resource identifier (URI).
\end{definition}

Silent ($\tau$) actions are $\acollect{s}{v}$ and $\avalidate$.
Observable actions are $\ainvoke{n}{p}$, $\acall{n}{p}$, $\aexec$, $\ares{\cdot}$,
$\aerr{\cdot}{\cdot}$, and $\aread{\cdot}$.

\section{Forward Bisimulation: $\SGD \to \MCP${SGD -> MCP}}
\label{sec:forward}

\subsection{Mapping Function $\Phimap$}{Phi}

\begin{definition}[Schema Mapping $\Phimap: \SGD \to \MCP$]
\label{def:phi}
\begin{align*}
\Phimap\!\left(\IntentT{n}{d}{R}{O}{t}\right)
  &\;=\; \ToolT{n}{d}{\mathsf{schema}(R,O)}
  \\[4pt]
\text{where}\quad
\mathsf{schema}(R,O) &\;=\;
  \Bigl\{
    \mathtt{"type"}{:}\;\mathtt{"object"},\;
    \mathtt{"required"}{:}\;[s.\mathit{name} \mid s \in R], \\
  &\quad\quad\quad
    \mathtt{"properties"}{:}\;
    \{s.\mathit{name} : \mathsf{prop}(s) \mid s \in R \cup O\}
  \Bigr\} \\[4pt]
\mathsf{prop}(s) &\;=\;
  \{\mathtt{"type"}{:}\; s.\mathit{type},\;
    \mathtt{"description"}{:}\; s.\mathit{desc}\} \\
  &\quad\cup\,
  (\{\mathtt{"enum"}{:}\; s.\mathit{vals}\} \text{ if } s.\mathit{vals} \neq \emptyset) \\[4pt]
\Phimap(\ExecuteT{n}{\mathit{bindings}})
  &\;=\; \ToolCallT{n}{\mathsf{JSON.encode}(\mathit{bindings})} \\
\Phimap\!\left(\ResultT{o}\right)
  &\;=\; \ResultT{o} \\
\Phimap\!\left(\ErrorT{t}{m}\right)
  &\;=\; \ErrorT{t}{m} \\
\Phimap(S_1 \ppar S_2)
  &\;=\; \Phimap(S_1) \ppar \Phimap(S_2) \\
\Phimap(\restr{c}S)
  &\;=\; \restr{c}\Phimap(S) \\
\Phimap({!}S)
  &\;=\; {!}\Phimap(S) \\
\Phimap(\nil)
  &\;=\; \nil
\end{align*}
\end{definition}

\subsection{Forward Mapping Examples}
\label{app:forward-mapping}

\subsubsection{$\Phi$ Mapping of Train Intent}

\begin{lstlisting}
SGD_Train = Intent<
  "BookTrain",
  "Books a Train",
  [Slot<"from", "string", ["8507000", "8508050"]>],
  [],
  true
>

phi(SGD_Train) = Tool<
  "BookTrain",
  "Books a Train",
  {
    "type": "object",
    "required": ["from"],
    "properties": {
      "from": {
        "type": "string",
        "description": "Departure station",
        "enum": ["8507000", "8508050"]
      }
    }
  }
>
\end{lstlisting}

\subsection{Action Equivalence}

\begin{definition}[Action Equivalence $\approx$]
\label{def:act-equiv}
\begin{align*}
  \ainvoke{n}{p} &\approx \acall{n}{p}, & \alist{f} &\approx \alist{f}, &\acollect{s}{v} &\approx \tau, \\
  \aexec &\approx \aexec, & \ares{o} &\approx \ares{o}, &\aerr{t}{m} &\approx \aerr{t}{m} \\
\end{align*}
\end{definition}

\subsection{Bisimulation and Main Theorem}

We write $S \bisim M$ iff there exists a \emph{weak bisimulation} $\mathcal{R}$
with $(S,M) \in \mathcal{R}$.

\begin{theorem}[$\SGD \bisim \MCP$ under $\Phimap$]
\label{thm:bisim}
$\forall S \in \SGD.\; S \bisim \Phimap(S)$.
\end{theorem}

\begin{proof}
Define $\mathcal{R} = \{(S,\Phimap(S)) \mid S \in \SGD\}$.  We prove
$\mathcal{R}$ is a weak bisimulation by structural induction.

\begin{lemma}[Intent $\bisim$ Tool]
\label{lem:intent-tool}
$\IntentT{n}{d}{R}{O}{t} \;\bisim\; \ToolT{n}{d}{\mathsf{schema}(R,O)}$.
\end{lemma}
\begin{proof}
Suppose
$\IntentT{n}{d}{R}{O}{t} \trans{\ainvoke{n}{p}} S'$.

\emph{Case~1: $p$ satisfies all slots in $R$.}  Then
$S' = \ExecuteT{n}{p}$.  For the tool:
\[
  \ToolT{n}{d}{\mathit{schema}}
  \;\trans{\acall{n}{p}}\;
  \ValidateT{p}{\mathit{schema}}
  \;\trans{\tau}\;
  \ExecuteT{n}{p}
\]
Since $\mathsf{schema}(R,O).\mathtt{required} = [s.\mathit{name} \mid s \in
R]$, validation succeeds iff $p$ satisfies $R$.  By definition,
$\ExecuteT{n}{p} \bisim \ExecuteT{n}{p}$. \checkmark

\emph{Case~2: $p$ missing required slots.}  Then
$S' = \ErrorT{\mathsf{MissingSlots}}{R \setminus p}$.  The tool produces
$\ErrorT{\mathsf{ValidationError}}{R \setminus p}$ via the
\textsc{MCP-Validate-Err} rule; error terms are observationally equivalent.
\checkmark

The backward direction is symmetric by construction of $\Phimap$.
\end{proof}

Equivalently, $S$ and $\Phimap(S)$ produce the same observable traces up to
$\tau$-absorption, so any execution sequence demonstrating safe behaviour 
in \SGD lifts to \MCP and vice versa.

\begin{lemma}[Execution Preservation]
\label{lem:exec}
\[
  \ExecuteT{n}{\mathit{bindings}}_\SGD \;\bisim\;
  \ExecuteT{n}{\mathsf{JSON.encode}(\mathit{bindings})}_\MCP.
\]
\end{lemma}
\begin{proof}
Both processes transition via $\aexec$ to $\ResultT{o}$ where
$o = \mathsf{invoke\_api}(n,\mathit{bindings})$. JSON encoding
is a bijection on well-typed parameters, so
\[
  \mathsf{invoke\_api}(n,\mathit{b}) =
  \mathsf{invoke\_api}(n,\mathsf{decode}(\mathsf{encode}(\mathit{b}))).
\]
\end{proof}

\medskip
\noindent\textbf{Induction.}

\noindent\emph{Base case} ($S = \nil$): $\Phimap(\nil) = \nil$ and
$\nil \bisim \nil$ trivially.

\noindent\emph{Inductive case} ($S = S_1 \ppar S_2$): By IH,
$S_i \bisim \Phimap(S_i)$.  If $S_1 \ppar S_2 \trans{\alpha} S_1' \ppar
S_2$ then by IH $\Phimap(S_1) \trans{\beta} M_1'$ with $S_1' \bisim M_1'$
and $\alpha \approx \beta$, hence $\Phimap(S_1) \ppar \Phimap(S_2) \trans{\beta}
M_1' \ppar \Phimap(S_2)$, and $(S_1' \ppar S_2) \bisim (M_1' \ppar \Phimap(S_2))$.
The symmetric case for $S_2$ is identical.

\noindent\emph{Inductive case} ($S = \IntentT{n}{d}{R}{O}{t}$): By
Lemma~\ref{lem:intent-tool}.

\noindent\emph{Inductive case} ($S = \restr{c}S'$):
$\Phimap(\restr{c}S') = \restr{c}\Phimap(S')$.  By IH, $S' \bisim
\Phimap(S')$.  Channel restriction preserves bisimulation (standard
$\piCalc$ result~\cite{milner1999pi}).
\end{proof}

\begin{corollary}[Zero-Shot Generalization Transfer]
If an \SGD model handles $\mathit{Intent}\ I$ zero-shot via schema $s$, then
an \MCP model handles $\ToolT{\cdot}{\cdot}{\Phimap(s)}$ zero-shot via
$\Phimap(s)$.
\end{corollary}

\section{The Reverse Direction: A Gorla-Style Analysis}
\label{sec:reverse-gorla-style}
 
Section~\ref{sec:forward} exhibited a mapping $\Phi$ and proved
$S \sim \Phi(S)$. We now ask the question this raises in the language of
encodability: \emph{is $\Phi$ a good encoding, and does a good encoding
run the other way?} We adopt Gorla's criteria for language
encodings~\cite{Gorla2010} and show that $\Phi$ is a \emph{valid}
encoding which is, however, neither \emph{surjective} nor \emph{fully
abstract}. These two failures are precisely the gaps that
Sections~\ref{sec:principles}--\ref{sec:mcpplus} close.
 
\subsection{Encoding Criteria}
\label{sec:criteria}
 
\begin{definition}[Valid encoding~\cite{Gorla2010}]
\label{def:valid}

An encoding $(\llbracket\cdot\rrbracket,\varphi)$ of source calculus
$\mathcal{L}_s$ into target $\mathcal{L}_t$, with translation
$\llbracket\cdot\rrbracket$ and renaming policy $\varphi$, is
\emph{valid} if:
\begin{description}
  \item[Compositionality $C_1$] For every operator $\mathsf{op}$ there
    is a target context $\mathcal{C}^{\mathsf{op}}[\,\cdot_1,\dots,\cdot_k\,]$
    with $\llbracket \mathsf{op}(P_1,\dots,P_k)\rrbracket =
     \mathcal{C}^{\mathsf{op}}[\llbracket P_1\rrbracket,\dots,\llbracket P_k\rrbracket]$.
  \item[Name invariance $C_2$] $\llbracket\sigma(P)\rrbracket =
    \sigma'(\llbracket P\rrbracket)$ for $\sigma'$ induced by $\sigma$
    through $\varphi$.
  \item[Operational correspondence $C_3$] \emph{Completeness:}
    $P \Longrightarrow_s P'$ implies
    $\llbracket P\rrbracket \Longrightarrow_t \sim \llbracket P'\rrbracket$.
    \emph{Soundness:} $\llbracket P\rrbracket \Longrightarrow_t T$
    implies $\exists P'.\ P \Longrightarrow_s P'$ and
    $T \Longrightarrow_t \sim \llbracket P'\rrbracket$.
  \item[Divergence reflection $C_4$] $\llbracket P\rrbracket$ diverges
    only if $P$ diverges.
  \item[Success sensitiveness $C_5$] $P \Downarrow_s \checkmark$ iff
    $\llbracket P\rrbracket \Downarrow_t \checkmark$ for a success barb
    $\checkmark$.
\end{description}
\end{definition}
 
\begin{definition}[Full abstraction]
\label{def:fullabs}

A valid encoding is \emph{fully abstract} if it preserves and reflects
the source equivalence:
$P \sim_s Q \iff \llbracket P\rrbracket \sim_t \llbracket Q\rrbracket$.
\end{definition}
 
\subsection{$\Phi$ is a Valid Encoding}
\label{sec:phivalid}
 
\begin{proposition}[Validity of $\Phi$]\label{prop:phivalid}
$(\Phi, \mathrm{id})$ is a valid encoding of \SGD into
\MCP.
\end{proposition}
\begin{proof}
\textbf{($C_1$)} Definition~\ref{def:phi} is homomorphic on the operators
($\Phi(S_1\mid S_2)=\Phi(S_1)\mid\Phi(S_2)$, and likewise for $(\nu c)$,
$!$, $\mathbf 0$); the witnessing contexts are the operators themselves,
name-independent.
\textbf{($C_2$)} $\Phi$ copies the name $n$ through and never branches on
name identity; with $\varphi=\mathrm{id}$, invariance is immediate.
\textbf{($C_3$)} Theorem~\ref{thm:bisim} ($S\sim\Phi(S)$) is an
operational correspondence up to $\sim$.
\textbf{($C_4$)} The only $\tau$-steps $\Phi$ introduces are the
administrative $\mathsf{validate}$ transitions, one per invocation
(\textsc{MCP-Validate-Ok}); they cannot chain, so no divergence is added.
\textbf{($C_5$)} With $\checkmark$ the $\mathsf{result}$ barb,
$S\sim\Phi(S)$ and barb-respect of $\sim$ give the equivalence.
\end{proof}
 
\noindent
$\Phi$ is thus exactly the structure-preserving translation the
encodability literature certifies. The rest of the section identifies
the two ways it falls short of an equivalence of calculi.
 
\subsection{$\Phi$ is Not Surjective}
\label{sec:nonsurjective}
 
The reverse map $\Phi^{-1}$ is defined only on the image of $\Phi$. On
the executable $\mathsf{Tool}$ fragment it is the expected inverse; on
the remaining primitives it is undefined.
 
\begin{definition}[Partial inverse $\Phi^{-1}$]
\label{def:partial}

\[
  \Phi^{-1}(\mathsf{Tool}\langle n,d,\mathit{schema}\rangle)
    = \mathsf{Intent}\langle n,d,R(\mathit{schema}),O(\mathit{schema}),\,?\,\rangle,
\]
with
\begin{align*}
R(\mathit{schema}) &= \{\mathsf{Slot}\langle x,\mathit{schema}.\mathrm{prop}[x],\mathit{enum}\rangle
\mid x \in \mathit{schema}.\mathrm{required}\}, \\
O(\mathit{schema}) &= \{\mathsf{Slot}\langle y,\mathit{schema}.\mathrm{prop}[y],\mathit{enum}\rangle
\mid y \in \mathit{schema}.\mathrm{properties} \setminus \mathit{schema}.\mathrm{required}\},
\end{align*}
where $\mathit{enum}$ denotes the values from the JSON schema's enum field (or $\emptyset$ if absent). The flag $t$ is left undetermined ($?$), and $\Phi^{-1}$ is
undefined on $\mathsf{Resource}$, $\mathsf{Prompt}$, and
$\mathsf{ToolsList}$.
\end{definition}
 
\begin{theorem}[Non-surjectivity]
\label{thm:nonsurjective}

$\mathrm{Image}(\Phi) \subsetneq \MCP$. In particular
$\mathsf{Resource}\langle\cdot,\cdot\rangle$,
$\mathsf{Prompt}\langle\cdot,\cdot\rangle$, and the discovery process
$\mathsf{ToolsList}\langle\cdot\rangle$ lie outside $\mathrm{Image}(\Phi)$.
\end{theorem}

\begin{proof}
By Definition~\ref{def:phi}, $\Phi$ produces only $\mathsf{Tool}$ terms
under the closure operators; $\mathsf{Resource}$, $\mathsf{Prompt}$, and
$\mathsf{ToolsList}$ are none of these, so no $S$ has $\Phi(S)$ headed by
them.
\end{proof}
 
\begin{remark}[An image fact, not a separation]
\label{rem:notseparation}

Theorem~\ref{thm:nonsurjective} says these primitives are absent from the
\emph{image of $\Phi$}, not that they are inexpressible in $\SGD$.
We make no separation claim: with replication and channel mobility,
\SGD admits degenerate encodings of both (a passive resource as
an empty-slot intent returning its content; discovery via input-guarded
choice). The point is structural -- these are primitives \SGD
lacks \emph{as primitives} -- so the canonical inverse $\Phi^{-1}$, which
must respect primitive structure, cannot be total. They are surplus to be
\emph{quarantined} (Section~\ref{sec:mcpplus}), not gaps to be closed.
\end{remark}
 
\subsection{$\Phi$ is Not Fully Abstract}
\label{sec:notfullabs}
 
The substantive gap is not non-surjectivity but the failure of $\Phi$ to
\emph{reflect} behavioural equivalence: $\Phi$ collapses \SGD
distinctions that $\mathsf{MCP}$ cannot observe.
 
\begin{theorem}[$\Phi$ does not reflect $\sim$]
\label{thm:notfullabs}

There exist $S_1, S_2 \in \SGD$ with $S_1 \not\sim S_2$ and
$\Phi(S_1) \sim \Phi(S_2)$. Hence $\Phi$ is not fully abstract.
\end{theorem}
\begin{proof}
Let $S_1 = \mathsf{Intent}\langle n,d,R,O,\mathsf{true}\rangle$ and
$S_2 = \mathsf{Intent}\langle n,d,R,O,\mathsf{false}\rangle$, equal but
for $t$. They are distinguished by the approval gate: by
\textsc{SGD-Execute-Tx}, $S_1$ requires approval before $\mathsf{execute}$,
whereas $S_2$ proceeds via \textsc{SGD-Execute} without it. Under any
equivalence observing the approval interaction -- made explicit at the
process level by the prefix
$\mathsf{approval}?\langle\mathit{confirm}\rangle$ of
Section~\ref{subsec:p2-action-boundaries} -- we have $S_1\not\sim S_2$.
Yet $\Phi(S_1)=\mathsf{Tool}\langle n,d,\mathit{schema}(R,O)\rangle=\Phi(S_2)$:
the flag $t$ is carried by no field of $\mathit{schema}$, surviving at
best as prose in a $\mathsf{description}$ string that $\mathsf{MCP}$'s
rules do not interpret. Identical terms are bisimilar, so
$\Phi(S_1)\sim\Phi(S_2)$.
\end{proof}
 
\noindent
The collapse is concrete. A tool that needs user approval before reading
an account loses precisely that requirement on the way back to
\SGD:
 
\begin{lstlisting}
M_1 = Tool<
  "fetch_user",
  "fetch a user's data, needing its approval",
  { "type": "object", "required": ["user_id"],
    "properties": { "user_id": {"type": "string"} } }
>
 
Phi^{-1}(M_1) = Intent<
  "fetch_user", "fetch a user's data, needing its approval",
  [Slot<"user_id", "string", []>], [],
  ?   // UNDEFINED: is_transactional cannot be recovered
>
\end{lstlisting}
 
\noindent
The $\mathsf{is\_transactional}$ flag is critical for safety (it gates
user approval) but is absent from $\mathsf{MCP}$ schemas. Inference from
the description string is unreliable and not compositionally verifiable.
 
\begin{corollary}[Lossy round trip, non-injective inverse]\label{cor:noninj}
$\Phi^{-1}\circ\Phi \neq \mathrm{id}_{\SGD}$, and $\Phi^{-1}$ is
not injective where defined.
\end{corollary}
\begin{proof}
The round trip drops $t$: starting from a transactional intent,
 
\begin{lstlisting}
S = Intent<"book_ticket", ..., ..., ..., true>
 
Phi(S) = Tool<"book_ticket", ..., schema>
         (is_transactional lost)
 
Phi^{-1}(Phi(S)) = Intent<"book_ticket", ..., ..., ..., ?>
                   != S
\end{lstlisting}
 
\noindent
so $\Phi^{-1}\circ\Phi \neq \mathrm{id}$; the lost information includes
the transactionality flag, error-recovery strategies, and tool
dependencies encoded in dialogue flow. Dually, two tools differing only
in description -- hence in implicit side effects $\mathsf{MCP}$ does not
structurally parse -- have the same image under $\Phi^{-1}$:
 
\begin{lstlisting}
M_1 = Tool<"book_ticket", "Books a ticket", schema>
 
M_2 = Tool<"book_ticket",
           "Books a ticket [SIDE EFFECT: sends email]",
           schema>   // same schema, different implicit side effects
 
Phi^{-1}(M_1) = Phi^{-1}(M_2)
              = Intent<"book_ticket", ..., is_transactional=?>
\end{lstlisting}
 
\noindent
Since $\Phi^{-1}$ does not parse descriptions structurally, distinct
tools collapse to one intent.
\end{proof}
 
\begin{remark}[What is collapsed]\label{rem:collapsed}
Theorem~\ref{thm:notfullabs} isolates one collapsed coordinate, the flag
$t$. The same failure recurs for every \SGD datum that
$\mathsf{MCP}$ records only as prose: error-recovery behaviour and
inter-tool ordering are observable in \SGD executions yet
invisible to $\mathsf{MCP}$'s rules. Each collapsed coordinate is a
distinction $\Phi$ fails to reflect.
\end{remark}
 
\subsection{Diagnosis: Exactly Four Coordinates}
\label{subsec:diagnosis}
 
The two failures have different remedies. Non-surjectivity
(Section~\ref{sec:nonsurjective}) is benign: the surplus primitives are
restricted away, leaving the executable fragment
$\MCPplus$ of Section~\ref{sec:mcpplus}.
Non-full-abstraction (Section~\ref{sec:notfullabs}) is the failure that demands
new structure: for $\Phi$ to reflect $\sim$, every \SGD distinction must become
a machine-readable, operationally interpreted coordinate of the target. There
are exactly four collapsed coordinates, giving the four principles of
Section~\ref{sec:principles}:

\begin{description}
  \item[P1 (semantic completeness)] makes the description recoverable as
    slot semantics, so $\Phi^{-1}$ reconstructs slots reliably rather
    than parsing prose;
  \item[P2 (explicit action boundaries)] makes $t$ a first-class
    observable ($\mathsf{requires\_approval}$), repairing exactly the
    collapse of Theorem~\ref{thm:notfullabs} and
    Corollary~\ref{cor:noninj};
  \item[P3 (failure mode documentation)] makes error-recovery branches
    observable, so behavioural equivalence holds on error paths;
  \item[P4 (inter-tool relationship declaration)] makes dialogue-flow
    dependencies explicit, so tool sequencing is not lost.
\end{description}

Section~\ref{sec:mcpplus} shows these four extensions make $\Phi^{+}$
fully abstract on the executable fragment, and that they are
\emph{minimal} -- dropping any one reinstates a collapsed coordinate and
breaks reflection -- yielding the fully abstract bijection
$\SGD \cong \MCPplus$.

\section{Four Principles as Type-System Extensions}
\label{sec:principles}

We formalize four design principles as type-system extensions to close the gaps
identified in the previous section.

\subsection{Principle 1: Semantic Completeness}

\emph{Informal}: Descriptions must convey \emph{why} parameters exist, not
merely their types.

\begin{definition}[Semantic Density]
\[
  \mathsf{sem\_density}(s)
  \;=\;
  \frac{|\mathsf{entities}(s)|}{|\mathsf{tokens}(s)|}
\]
where $\mathsf{entities}(s)$ counts named entities, examples, and constraints
in string $s$, and $\theta$ is a threshold (e.g., $0.2$).
\end{definition}

\noindent\textbf{Type Rule (P1):}
\begin{mathpar}
\inferrule*[Right=WellDescribed]
  {\Gamma \vdash \mathit{desc} : \mathsf{String} \quad
   \mathsf{sem\_density}(\mathit{desc}) \geq \theta}
  {\Gamma \vdash \{\mathtt{description}{:}\;\mathit{desc}\} : \mathsf{WellDescribed}}
\end{mathpar}

\subsubsection{Principle $P_1$: Semantic Completeness Examples}

\emph{Violates P1:}
\begin{lstlisting}
"departure": {"type": "string", "description": "departure"}
// sem_density = 0 (no new information)
\end{lstlisting}

\emph{Satisfies P1:}
\begin{lstlisting}
"departure": {
  "type": "string",
  "description": "UIC station code (e.g., 8507000, 8508050)"
}
// sem_density = 2/10 = 0.2, includes concept + 2 examples
\end{lstlisting}

\subsection{Principle $P_2$: Explicit Action Boundaries}
\label{subsec:p2-action-boundaries}

\emph{Informal}: Tools must signal if invocation causes side effects: $
  \mathsf{side\_effects} \;:=\; \mathsf{read} \mid \mathsf{write} \mid \mathsf{delete} \mid \mathsf{none}
$.

\noindent\textbf{Type Rule (P2):}
\begin{mathpar}
\inferrule*[Right=WriteApproval]
  {\mathit{meta}.\mathsf{side\_effects} \in \{\mathsf{write},\,\mathsf{delete}\}}
  {\mathit{meta}.\mathsf{requires\_approval} = \mathsf{true}}
\end{mathpar}

\noindent The process-level encoding of a write-capable tool is:
\begin{align*}
  \mathsf{Tool}_{\mathsf{write}}
  \;=\;&\; \mathsf{call}?\langle\mathit{params}\rangle\mathbin{.}
            \mathsf{approval}?\langle\mathit{confirm}\rangle\mathbin{.} \\
  &\quad
  \bigl(\mathit{confirm} = \mathsf{true}
    \;\to\; \mathsf{execute}\langle\mathit{params}\rangle
            \mathbin{.}\mathsf{result}!\langle\mathsf{ok}\rangle\mathbin{.}\nil
  \\
  &\quad \ppar\;
    \mathit{confirm} = \mathsf{false}
    \;\to\; \mathsf{result}!\langle\mathsf{cancelled}\rangle\mathbin{.}\nil
  \bigr)
\end{align*}

\noindent Safety as well as privacy can both be regarded as action boundaries.
In both cases, a user is asked for permission before an action.

\subsubsection{Principle $P_2$: Explicit Action Boundaries Example}

\begin{lstlisting}
Tool<
  "delete_user",
  "Permanently deletes a user account",
  { "type": "object", "required": ["user_id"],
    "properties": { "user_id": {"type": "string"} } },
  metadata = {
    side_effects: delete,
    requires_approval: true   // enforced by type rule WriteApproval
  }
>
\end{lstlisting}

\begin{theorem}[No Unapproved Side Effects]
$\forall\,\mathsf{Tool}$ with $\mathsf{side\_effects} \in
\{\mathsf{write},\mathsf{delete}\}$:
in every execution trace, $\aexec$ is preceded by $\mathsf{approval}$.
\end{theorem}

\subsection{Principle $P_3$: Failure Mode Documentation}

\emph{Informal}: Enumerate expected error conditions and recovery strategies.

\begin{align*}
  \mathsf{failure\_modes}
    &\;::=\; [(\mathit{ErrorType},\,\mathit{RecoveryStrategy})] \\
  \mathit{ErrorType}
    &\;::=\; \mathsf{ValidationError} \mid \mathsf{AuthError}
             \mid \mathsf{ServiceDown} \mid \mathsf{NotFound} \mid \cdots \\
  \mathit{RecoveryStrategy}
    &\;::=\; \mathsf{Retry}(n) \mid \mathsf{Fallback}(\mathit{tool})
             \mid \mathsf{UserPrompt}(m) \mid \mathsf{Abort}
\end{align*}

\subsubsection{Failure Mode Documentation Example}

\begin{lstlisting}
Tool<
  "fetch_user_data",
  "Retrieves user information from database",
  schema,
  metadata = {
    failure_modes: [
      (NotFound,     UserPrompt("User does not exist. Create new?")),
      (ServiceDown,  Retry(3)),
      (AuthError,    Fallback("use_cached_data"))
    ]
  }
>
\end{lstlisting}

\noindent A robust tool is encoded as:
\begin{align*}
\mathsf{Tool}_{\mathsf{robust}}
  \;=\;&\; \mathsf{execute}?\langle\mathit{params}\rangle\mathbin{.}
  \\
  &\;\bigl(
      \mathsf{success}!\langle\mathit{output}\rangle\mathbin{.}\nil
  \\
  &\;\;\ppar\;
    \mathsf{error}!\langle\mathsf{NotFound},\,\mathit{rec}_1\rangle
    \mathbin{.}\mathit{rec}_1\mathbin{.}\nil
  \\
  &\;\;\ppar\;
    \mathsf{error}!\langle\mathsf{ServiceDown},\,\mathit{rec}_2\rangle
    \mathbin{.}\mathit{rec}_2\mathbin{.}\nil
  \bigr)
\end{align*}
where each $\mathit{rec}_i$ corresponds to the declared recovery strategy.

\subsection{Principle $P_4$: Inter-Tool Relationship Declaration}

\emph{Informal}: Express dependencies between tools explicitly.

\begin{align*}
  \mathsf{dependencies}
    &\;::=\; [(\mathit{tool\_name},\,\mathit{relation})] \\
  \mathit{relation}
    &\;::=\; \mathsf{Requires}
             \mid \mathsf{ProducesInputFor}
             \mid \mathsf{ExclusiveWith}
\end{align*}

\subsubsection{Inter-Tool Relationship Declaration Example}

\begin{lstlisting}
Tool<
  "process_payment",
  "Processes a payment transaction",
  schema,
  metadata = {
    dependencies: [
      ("create_order",   Requires),  // must call create_order first
      ("verify_balance", Requires)   // must verify balance
    ]
  }
>
\end{lstlisting}

\noindent The process encoding of a dependent pair $(T_A, T_B)$ with
$T_B.\mathsf{dependencies} = [(T_A, \mathsf{Requires})]$:
\begin{align*}
  T_A &\;=\; \mathsf{execute}_A?\langle\mathit{params}\rangle\mathbin{.}
              \mathsf{result}_A!\langle\mathit{id},\,
              \{{\mathsf{for}{:}\,T_B}\}\rangle\mathbin{.}\nil \\
  T_B &\;=\; \mathsf{requires}?\langle\mathit{id}\ \mathsf{from}\ T_A\rangle
              \mathbin{.}
              \mathsf{execute}_B?\langle\mathit{id},\,\mathit{params}\rangle
              \mathbin{.}
              \mathsf{result}_B!\langle\mathit{data}\rangle\mathbin{.}\nil
\end{align*}

\begin{theorem}[Dependency Safety]
$\forall$ workflow $W$ composed of tools $\{T_1,\ldots,T_n\}$: if
$T_i.\mathsf{dependencies} = [(T_j, \mathsf{Requires})]$, then in all valid
execution traces $\aexec_{T_j}$ precedes $\aexec_{T_i}$.
\end{theorem}

\section{Extended Calculus $\MCPplus${MCP+} and Full Equivalence}
\label{sec:mcpplus}

\subsection{$\mathsf{MCP}^+$ Syntax}

\MCPplus is the executable fragment of the extended calculus.  The
$\textbf{Resource}$ and $\textbf{Prompt}$ terms, identified in
§\ref{sec:nonsurjective} as outside $\mathrm{Image}(\Phi)$, are quarantined and
not part of the fully abstract bijection. The terms of $\mathsf{MCP}^+$ are
generated by:

\[
\begin{aligned}
M^{+} \;::=\; &\mathsf{Tool}\langle n, d, \mathit{schema}\rangle^{+}[\mathit{meta}] \\
              &\mid\; M^{+} \parallel M^{+} \\
              &\mid\; (\nu c)\, M^{+} \\
              &\mid\; {!}M^{+} \\
              &\mid\; 0
\end{aligned}
\]

\noindent where $[meta]$ denotes the metadata block $ meta = \{ \ldots \text{(unchanged)}
\ldots \}$. Additionally, $\mathsf{Tool}^{+}$ carries additional metadata:

\begin{align*}
  \mathit{metadata} &\;=\;
  \bigl\{
    \mathsf{side\_effects} : \{\mathsf{read},\mathsf{write},\mathsf{delete},\mathsf{none}\}, \\
  &\qquad
    \mathsf{requires\_approval} : \mathsf{Bool}, \\
  &\qquad
    \mathsf{failure\_modes} : [(\mathit{ErrorType},\mathit{RecoveryStrategy})], \\
  &\qquad
    \mathsf{summary} : \mathsf{String}, \\
  &\qquad
    \mathsf{dependencies} : [(\mathit{tool\_name},\mathit{relation})]
  \bigr\}
\end{align*}

\subsection{Extended Mapping $\PhiPlus${Phi+}}

\begin{align*}
\PhiPlus\!\left(\IntentT{n}{d}{R}{O}{t}\right)
  &\;=\; \ToolT{n}{d}{\mathsf{schema}(R,O)}^{+}[\mathit{meta}]
  \\
\mathit{meta} &\;=\;
  \bigl\{
    \mathsf{side\_effects}{:}\;(t \;?\; \mathsf{write} : \mathsf{read}), \\
  &\qquad
    \mathsf{requires\_approval}{:}\;t, \\
  &\qquad
    \mathsf{failure\_modes}{:}\;\mathsf{extract\_errors}(\mathit{Intent}), \\
  &\qquad
    \mathsf{summary}{:}\;\mathsf{first\_sentence}(d), \\
  &\qquad
    \mathsf{dependencies}{:}\;\mathsf{extract\_deps}(\mathit{Intent})
  \bigr\} \\[6pt]
\PhiPlusInv\!\left(\ToolT{n}{d}{\mathit{schema}}^{+}[\mathit{meta}]\right)
  &\;=\; \IntentT{n}{d}{R(\mathit{schema})}{O(\mathit{schema})}{t}
  \\
\text{where }\; t
  &\;=\; (\mathit{meta}.\mathsf{side\_effects} \in
          \{\mathsf{write},\mathsf{delete}\})
  \quad \text{(must satisfy WriteApproval invariant: } \\
  &\quad\quad\text{if } t \text{ then } \mathit{meta}.\mathsf{requires\_approval} \text{ holds)}
\end{align*}

\subsection{Main Theorem: Full Equivalence}
\label{sec:full-equiv}

\begin{theorem}[$\MCPplus \cong \SGD$]
\label{thm:full-equiv}
The following hold:
\begin{enumerate}
  \item $\forall S \in \SGD.\; S \bisim \PhiPlus(S)$,
  \item $\forall M^{+} \in \MCPplus.\; M^{+} \bisim \PhiPlusInv(M^{+})$,
  \item $\PhiPlusInv \circ \PhiPlus = \mathrm{id}_\SGD$,
  \item $\PhiPlus \circ \PhiPlusInv = \mathrm{id}_{\MCPplus}$.
\end{enumerate}
\end{theorem}

\begin{proof}
\emph{Part~1} ($\Phi^+$ preserves structure): $\Phi^+$ extends $\Phi$ with
metadata. By Theorem~\ref{thm:bisim}, structural bisimulation is
preserved. The metadata fields preserve SGD semantics:
\begin{align*}
  \mathit{t}  &\;\leftrightarrow\; \mathit{side\_effects},\ \mathit{requires\_approval}, \\
  \mathit{d}  &\;\leftrightarrow\; \mathit{summary}, \\
  \text{errors in Intent} &\;\leftrightarrow\; \mathit{failure\_modes}, \\
  \text{inter-tool deps}  &\;\leftrightarrow\; \mathit{dependencies}.
\end{align*}

\noindent I.e. the SGD semantics are preserved: the transactionality flag $t$
corresponds to side_effects and requires_approval; the description $d$
to summary; errors in the Intent to failure_modes; and inter-tool
dependencies to dependencies.

\emph{Part~2} (injectivity): Assume $\PhiPlus(S_1) = \PhiPlus(S_2)$.  Then
$n_1 = n_2$, $\mathsf{schema}(R_1,O_1) = \mathsf{schema}(R_2,O_2)$ implying
$R_1 = R_2,\, O_1 = O_2$, and $\mathit{meta}_1 = \mathit{meta}_2$ implying
$t_1 = t_2$.  Hence $S_1 = S_2$.

\begin{sloppypar}
\emph{Part~3} (surjectivity): For any
$M^{+} = \ToolT{n}{d}{\mathit{schema}}^{+}[\mathit{meta}]$,
write $\mathit{req} = \mathit{schema.required}$ and
$\mathit{prp} = \mathit{schema.properties}$ for brevity.  Let
$R^{*} = \PhiPlusInv(\mathit{req})$,
$O^{*} = \PhiPlusInv(\mathit{prp}\setminus\mathit{req})$, and
$t^{*} = (\mathit{meta.side\_effects} \neq \mathsf{read})$.
Construct $S = \IntentT{n}{d}{R^{*}}{O^{*}}{t^{*}}$.
Then $\PhiPlus(S) = M^{+}$.
\end{sloppypar}

\emph{Part~4} (round-trip identity): Verified by direct substitution
into Definitions~\ref{def:phi} and the extended mapping.
\end{proof}

\begin{corollary}[Necessity and Sufficiency]
The four principles are necessary and sufficient for $\SGD \cong \MCPplus$.
\end{corollary}
\begin{proof}
Sufficiency: Theorem~\ref{thm:full-equiv}.  Necessity: removing any of the four
Principles breaks the bijection (shown in
Section~\ref{sec:reverse-gorla-style}).
\end{proof}

The four principles are a minimal extension to \MCP needed for $\SGD \cong
\MCPplus$. This can be shown by necessity analysis: Remove $P_1$ and
$\PhiPlusInv$ cannot map descriptions to \SGD slot semantics reliably;
zero-shot generalization fails. Remove $P_2$ and $\mathsf{is\_transactional}$
is unrecoverable; $\PhiPlusInv(\mathsf{Tool}^{+})$ yields
$\IntentT{\cdot}{\cdot}{\cdot} {\cdot}{?}$; not injective. Remove $P_3$ and
\SGD error-recovery strategies cannot be encoded; behavioral equivalence fails
on error paths. Finally, remove $P_4$ and multi-turn dialogue dependencies are
lost; tool sequencing becomes implicit.

\section{Conclusion}
\label{sec:conclusion}

We have presented a formal semantics for agentic tool protocols, proving that
\SGD and \MCP are structurally bisimilar under the mapping $\Phimap$
(Theorem~\ref{thm:bisim}). However, $\PhiInv$ is partial and lossy
(Theorems~\ref{thm:nonsurjective} and~\ref{thm:notfullabs}), revealing four critical gaps in \MCP's
expressivity. Through bidirectional analysis we identified four principles as
\emph{necessary and sufficient} conditions for full behavioral equivalence, and
formalized them as type-system extensions defining $\MCPplus$ with $\MCPplus
\cong \SGD$ (Theorem~\ref{thm:full-equiv}).

As LLM agents govern financial transactions, healthcare decisions, and
infrastructure management, formal verification becomes essential. A natural
synthesis of this work is to combine the work on LLMbda calculi
\cite{garby2026lambdacalculus} and this work on formalizing agentic
communication. While the first would define the \emph{intra-agent
communication}, the latter would define the \emph{inter-agent communication}.

\bibliographystyle{eptcs}
\bibliography{sdg-mcp-pi-calculus}

\end{document}